\documentclass[conference]{IEEEtran}
\IEEEoverridecommandlockouts

\newcommand{\vx}{\mathbf{x}}
\newcommand{\vy}{\mathbf{y}}

\newcommand{\vt}{\mathbf{t}}
\newcommand{\vi}{\mathbf{i}}

\newcommand{\vI}{\mathbf{I}}
\newcommand{\vT}{\mathbf{T}}

\newcommand{\vtheta}{\boldsymbol{\theta}}
\newcommand{\vdelta}{\boldsymbol{\delta}}
\newcommand{\vmu}{\boldsymbol{\mu}}

\newcommand{\mF}{\mathbf{F}}

\newcommand{\mH}{\mathbf{H}}
\newcommand{\mK}{\mathbf{K}}
\newcommand{\mR}{\mathbf{R}}
\newcommand{\mQ}{\mathbf{Q}}
\newcommand{\mO}{\mathbf{O}}
\newcommand{\R}{\mathbb{R}}
\newcommand{\E}{\mathbb{E}}
\newcommand{\B}{\mathcal{B}}
\newcommand{\mI}{\mathbb{I}}
\newcommand{\gN}{\mathcal{N}}
\newcommand{\D}{\mathcal{D}}
\newcommand{\mA}{\mathbf{A}}

\newcommand{\mS}{\mathcal{S}}
\newcommand{\mSigma}{\mathbf{\Sigma}}

\newcommand{\mU}{\mathbf{U}}
\newcommand{\mV}{\mathbf{V}}
\newcommand{\mC}{\mathbf{C}}

\usepackage{cite}
\usepackage{natbib}
\usepackage{amsmath,amssymb,amsfonts}
\usepackage{algorithm}
\usepackage{amsthm}
\usepackage{algpseudocode}
\usepackage{multirow}
\usepackage{booktabs} 
\usepackage{graphicx}
\usepackage{textcomp}
\usepackage{xcolor}

\newtheorem{theorem}{Theorem}[section]
\newtheorem{lemma}[theorem]{Lemma}
\newtheorem{proposition}[theorem]{Proposition}

\def\BibTeX{{\rm B\kern-.05em{\sc i\kern-.025em b}\kern-.08em
    T\kern-.1667em\lower.7ex\hbox{E}\kern-.125emX}}
\begin{document}

\title{Bayesian Natural Gradient Fine-Tuning of CLIP Models via Kalman Filtering\\
}

\author{

\IEEEauthorblockN{Hossein Abdi}
\IEEEauthorblockA{
\textit{The University of Manchester}\\
hossein.abdi@manchester.ac.uk}
\and
\IEEEauthorblockN{Mingfei Sun}
\IEEEauthorblockA{
\textit{The University of Manchester}\\
mingfei.sun@manchester.ac.uk}
\and
\IEEEauthorblockN{Wei Pan}
\IEEEauthorblockA{
\textit{The University of Manchester}\\
wei.pan@manchester.ac.uk}
}

\maketitle

\begin{abstract}
Vision-language pre-trained models, such as CLIP, have established new benchmarks in multimodal data mining. In such models, few-shot fine-tuning is a major challenge to achieve optimal performance on both in-distribution (ID) and out-of-distribution (OOD) datasets, especially when labeled data is scarce. Most existing fine-tuning approaches rely on first-order gradient-based optimizers, which typically suffer from slow convergence, sensitivity to step-size hyperparameters, and poor generalization in OOD settings. In contrast, second-order methods utilize local curvature information of the loss landscape to adjust the update step size. This is particularly beneficial for CLIP models, whose non-convex loss functions often contain sharp critical points. In such cases, natural gradient direction can offer more substantial and efficient per-iteration updates when fine-tuning with limited data. Natural Gradient Descent (NGD) is obtained by preconditioning the standard gradient with the inverse Fisher Information Matrix (FIM), which is computationally expensive for large models. To address this, we propose a Bayesian approximation of NGD using a Kalman filter for CLIP models. Our method combines the benefits of second-order optimization with Bayesian inference, which enhances generalization while providing uncertainty quantification.
Extensive experiments conducted on diverse image classification datasets demonstrate that our algorithm consistently achieves superior--or comparable--ID performance and improved OOD robustness compared to state-of-the-art baselines. To the best of our knowledge, this work represents the first successful application of Kalman filtering to fine-tuning CLIP-based models, which enables more robust and efficient learning in vision-language tasks.
\end{abstract}


\section{Introduction}
\label{introduction}
Pre-trained vision-language models, particularly CLIP \citep{radford2021learning}, have demonstrated remarkable performance in zero-shot and few-shot multimodal learning tasks. However, optimal performance typically requires further adaptation to specific tasks. The effectiveness of transferring knowledge from such pre-trained models is highly sensitive to the distributional alignment between the pre-training (source) data and the task-specific (target) data. In real-world deployments, significant domain shifts can lead to substantial performance degradation.

\begin{figure}[t]
  \centering
  \includegraphics[width=1.0\linewidth]{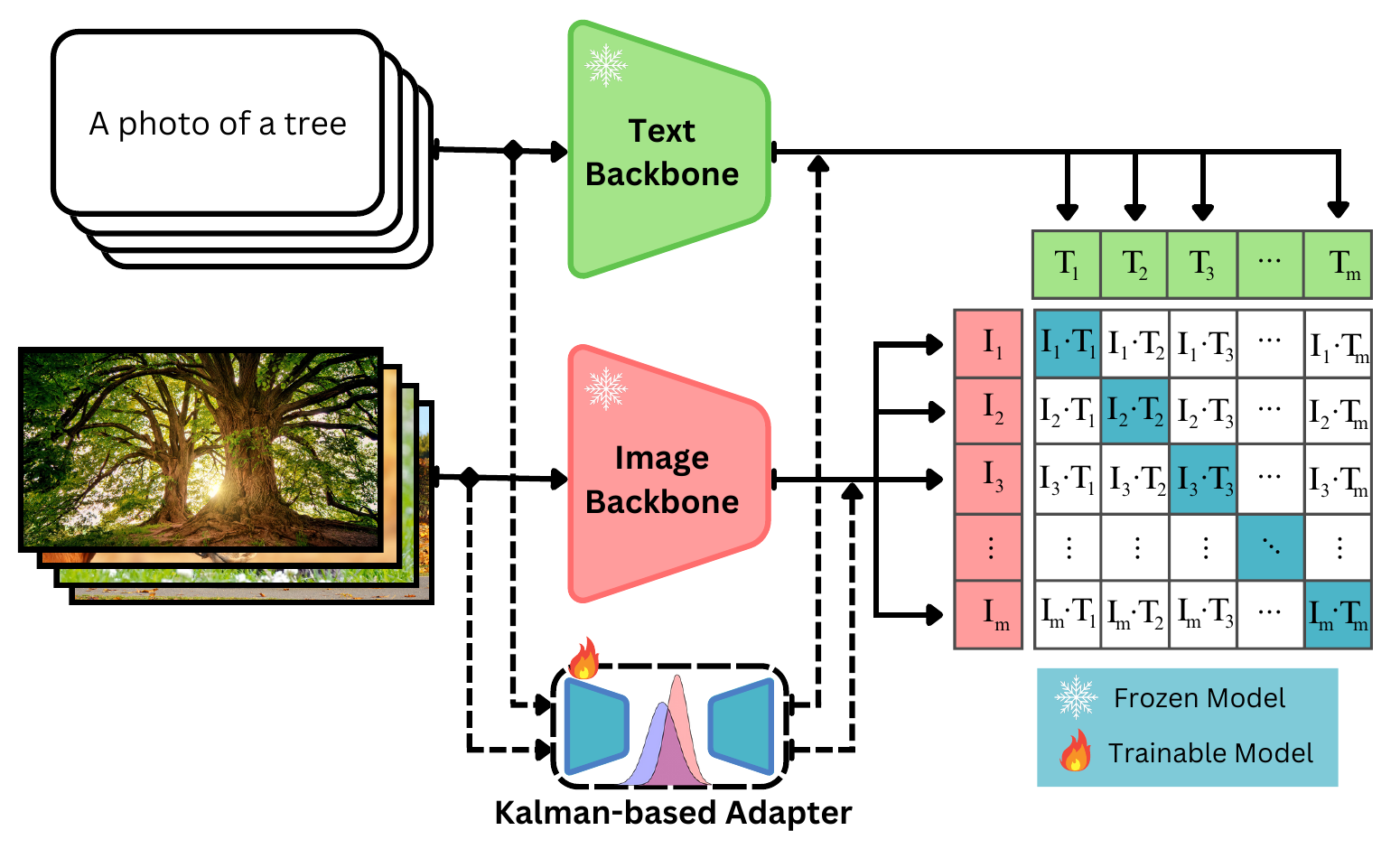}
  \caption{We employ a Kalman-based adapter to fine-tune the CLIP models. Kalman-based optimization algorithm closely approximates the natural gradient direction within a Bayesian framework. While natural gradient facilitates improved ID performance, Bayesian formulation inherently enables uncertainty quantification, which leads to improvement in OOD generalization.}
  \vspace{-0.7cm}
  \label{fig: Adapter}
\end{figure}

CLIP-based fine-tuning approaches have been proposed to improve the In-Distribution (ID) performance and enhance the Out-of-Distribution (OOD) generalization. However, identifying an optimal strategy to balance ID performance with OOD generalization remains an open and critical research question \citep{kumar2022fine}.
Most of the existing methods predominantly rely on first-order gradient-based optimizers such as SGD, Adam, and their variants, which may result in (sub-)optimal convergence in non-convex loss landscapes, particularly under distribution shifts \citep{reddi2019convergence, foret2020sharpness}.

To address these challenges, \emph{second-order} optimization methods have emerged as sophisticated alternatives to first-order approaches. These methods exploit curvature information from the loss manifold, typically via the Hessian matrix (as in Newton’s method or quasi-Newton methods \citep{nocedal1999numerical}) or the Fisher information matrix (as in Natural Gradient Descent \citep{amari1998natural}). By incorporating curvature, second-order optimizers dynamically adjust updates based on the local geometry of the loss landscape. This curvature-aware approach is particularly effective in few-shot learning, where sharp minima in the loss landscape are common. Second-order methods navigate these regions more effectively and boost ID performance \citep{foret2020sharpness, keskar2016large}.

Furthermore, \emph{Bayesian} methods have shown superior effectiveness in managing uncertainty and enhancing the robustness of OOD compared to gradient-based techniques \citep{wilson2020bayesian, ovadia2019can}. These methods are highly effective at capturing model uncertainty and incorporating prior knowledge, which are critical to improving generalization to unseen data. Among Bayesian approaches, Kalman filtering, a step-size-independent and free from gradients algorithm, has shown considerable potential in machine learning optimization. In particular, it can operate as a straightforward second-order optimizer within a Bayesian framework \citep{ollivier2018online}, which is expected to provide effective convergence while maintaining robustness to distributional shifts. 



Building on these foundations, we propose a robust fine-tuning algorithm based on Bayesian inference with the Kalman filtering method to simultaneously improve the ID performance and enhance the OOD robustness. 
%
%
Our contributions can be summarized as follows:
\begin{itemize}
    \item We introduce a robust adapter for fine-tuning CLIP-based vision-language models by leveraging Bayesian inference with Kalman filtering. This approach achieves consistent improvements in ID performance while improving the generalization of OOD.
    \item To the best of our knowledge, this work represents the first successful application of the Kalman filter algorithm for fine-tuning CLIP-based vision-language models. Our method delivers robust and effective performance across both ID and OOD scenarios.
    \item We conducted extensive experiments in multimodal classification tasks to demonstrate that our algorithm consistently achieves superior ID performance and OOD generalization compared to existing baselines.
    
\end{itemize}

\section{Related Work}
\label{related_work}

\paragraph{Few-Shot Fine-Tuning Methods for CLIP}

To overcome the limitations of zero-shot learning, numerous studies have explored few-shot learning with CLIP. CoOp (Context Optimization) \citep{zhou2022learning} introduces learnable vectors as text prompts to fine-tune CLIP for specific tasks. The CLIP-Adapter \citep{gao2024clip} enhances the performance by employing an adapter-based fine-tuning approach, which adapts the model to downstream tasks with minimal task-specific data. Tip-Adapter-F \citep{zhang2022tipadaptertrainingfreeadaptionclip} utilizes a key-value cache model from a few-shot training set to update CLIP's prior knowledge through feature retrieval. ClipFit \citep{li2024visionlanguagemodelfinetuningsimple} fine-tunes specific parameters, such as bias terms and normalization layers, to improve zero-shot performance while avoiding catastrophic forgetting. SAFT (Sparse Adaptation for Fine-Tuning) \citep{nguyen2024saftoutofdistributiongeneralizationfinetuning} updates only a small subset of parameters with large gradient magnitudes, which preserves general knowledge and improves OOD generalization. SAFE \citep{zhu2024enhancing} fine-tunes the attention pooling layer of CLIP's visual encoder to improve performance in few-shot scenarios by focusing on task-specific semantics. The mask-aware fine-tuning method (MAFT) \citep{jiao2023learningmaskawarecliprepresentations} addresses CLIP's insensitivity to mask proposals in zero-shot segmentation by fine-tuning a specialized encoder while maintaining transferability. Domain-Aligned CLIP (DAC) \citep{gondal2024domain} enhances both intra-modal and inter-modal alignment without modifying CLIP's parameters, while GLCM-Adapter \citep{wang2024glcm} improves few-shot learning by considering both global and local views of the input image for more robust recognition.
Despite these advancements, the development of a fine-tuning method to balance ID performance with OOD generalization remains an open research area.

\begin{algorithm}[t]
\caption{Our Kalman-based Algorithm}
\label{alg:algorithm1}
\begin{algorithmic}[1]
\State \textbf{Initialization:}\\
    $p(\vtheta_0) = \gN(\vmu_0,\mSigma_0)$,   $\mR_{0} = \mO_{m \times m} + \epsilon \mI$\\
    \For{$k = 1, 2, ... $}
        \State \textbf{Prediction:}
        \State $\vmu_{k|k-1} = \vmu_{k-1}$
        \State $\mSigma_{k|k-1} = \mSigma_{k-1} + \mQ_{k}$
        \State
        \State \textbf{Pre-Updating:}
        \State $\hat{\vy}_k = h(\B_{k}, \vmu_{k \mid k-1})$
        \State $\boldsymbol{H}_k = \nabla_{\vtheta}h|_{(\B_{k}, \vmu_{k \mid k-1})}$
        \State $d_M = \sqrt {\left(\vy_k - \hat{\vy}_k\right) \boldsymbol{R}_{k-1}^{-1} \left(\vy_k - \hat{\vy}_k\right)^{\top}}$
        \State $\lambda = \mathbf{e}^{-\alpha d_M}$
        \State $\hat{\mR}_k = \left(\vy_k - \hat{\vy}_k \right)\left(\vy_k - \hat{\vy}_k \right)^{\top} + \mH_k \mSigma_{k|k-1} \mH_k^{\top}$
        \State $\boldsymbol{R}_k = \beta \boldsymbol{R}_{k-1} + \lambda (1-\beta) \hat{\boldsymbol{R}}_k $
        \State
        \State \textbf{Updating:}
        \State $ \mK_k = \mSigma_{k|k-1} \mH_k^{\top} \left(\mH_k \mSigma_{k|k-1} \mH_k^{\top} + \mR_k\right)^{-1}$
        \State $\vmu_{k} = \vmu_{k|k-1} + \lambda \mK_k (\vy_k - \hat{\vy}_k)$
        \State  $\mSigma_{k} = \mSigma_{k|k-1} - \mK_k \mH_k \mSigma_{k|k-1}$
        \State
        \State \textbf{Output:}
        \State Posterior: $p(\vtheta_k \mid \B_{1:k}) = \gN(\vmu_k,\mSigma_k)$
    \EndFor
\end{algorithmic}
\end{algorithm}

\paragraph{Kalman Filter for Optimizing Neural Networks}
The concept of employing the Kalman filter for parameter optimization in deep learning originates from the work of Singhal \citep{singhal1988training}, who demonstrated that training neural networks can be framed as a system identification problem for nonlinear dynamic systems. This insight led to the use of the Extended Kalman Filter (EKF) to train neural network parameters. The superior performance of Kalman-based training algorithms over traditional backpropagation methods sparked significant interest in exploring the connections between these two classical approaches \citep{ruck1992comparative, ollivier2018online}. 
To enhance the applicability of the Kalman filter to large-scale models, several studies have focused on reducing its computational complexity. A notable approach involves the use of matrix partitioning techniques \citep{shah1992optimal, puskorius1991decoupled} and a low-dimensional (block-)diagonal approximation of the covariance matrix \citep{murtuza1994node}. 
More recently, Ollivier in \citep{ollivier2018online, ollivier2019extended} established that training with a Kalman filter is equivalent to a second-order optimizer within a Bayesian framework.
This finding renewed interest in this training method once again. 
Subsequent studies have further addressed computational challenges, such as the diagonal Gaussian approximation introduced in \citep{chang2022diagonal, abdi2024loko} and the low-rank plus diagonal decomposition of the posterior precision matrix proposed in \citep{chang2023low}. Furthermore, \citep{hennig2024computation} developed a matrix-free iterative algorithm to improve efficiency, while Gomez in \citep{gomez2021decoupled} introduced a decoupled EKF (DEKF) for factorization models. 
The Kalman filter has also found applications in specialized domains, including continual learning \citep{titsias2023kalman}, test-time adaptation \citep{schirmer2024test}, and reinforcement learning \citep{shashua2020kalman, totaro2021fast, shih2024fast}. Other notable advancements include loss-adaptive Kalman optimization \citep{davtyan2022koala}, the Bayesian online natural gradient \citep{jones2024bayesian}, and methods for handling non-stationary data in online learning \citep{jones2022learning, jones2022neural}.  
Despite these advancements, the application of the Kalman algorithm to CLIP-based vision-language models remains unexplored, which presents a significant gap in the literature.

\section{Preliminaries and Background}
\label{sec:Preliminaries and Background}

\subsection{CLIP-based Vision-Language Models}

CLIP (Contrastive Language–Image Pretraining) is a vision-language model developed by OpenAI that learns to associate images and texts through a contrastive learning framework.
Consider a dataset of vision-language pairs in the form of $(image,text)$ which is drawn i.i.d from the source distribution $p_{s}$, $\D_s = \{\vx_j=(\vx^{\vi}_j, \vx^{\vt}_j) \sim p_{s}(\vx) \}_{j=1}^{N_s}$.
The objective is to train a model capable of aligning semantically corresponding image-text pairs in a shared latent space while ensuring that non-matching pairs are distinctly separated. To this end, two encoders $h_{image}$ and $h_{text}$ are reparameterized by $\vtheta = \text{vec}(\vtheta^{\vi} ,\vtheta^{\vt}) \in \R^n$ to project images and text, respectively, into a $d-$dimensional embedding space. The resulting embeddings $\vI_j=h_{image}(\vx^{\vi}_j, \vtheta^{\vi})$ and $\vT_j=h_{text}(\vx^{\vt}_j, \vtheta^{\vt})$ are optimized such that their similarity reflects the semantic correspondence between the paired image and the text $(\vx^{\vi}_j, \vx^{\vt}_j)$. This is achieved by maximizing the diagonal entries of the cosine similarity matrix $\mS_C(\vI_j , \vT_{j'}) \text{ for } j=j'$ and minimizing its off-diagonal entries $\mS_C(\vI_j , \vT_{j'}) \text{ for } j \neq j'$ over a batch of data.
Formally, the loss function is defined as:
\begin{equation} \label{Eq. CLIP loss}
    \begin{split}
        \mathcal{L}_{\text{CLIP}} = -\frac{1}{m} \sum_{i=1}^m \Bigg[ & \log \frac{\exp(\mS_C(\vI_i , \vT_i) / \tau)}{\sum_{j=1}^m \exp(\mS_C(\vI_i , \vT_j) / \tau)} \\
        & + \log \frac{\exp(\mS_C(\vI_i , \vT_i) / \tau)}{\sum_{j=1}^m \exp(\mS_C(\vI_j , \vT_{i}) / \tau)} \Bigg],
    \end{split}
\end{equation}
where the cosine similarity matrix is given by $\mS_C(\vI , \vT) = \frac{\vI \cdot \vT}{\|\vI\|\|\vT\|}$, and $\tau$ represents the temperature parameter.

\subsection{Natural Gradient Descent}

For the negative log loss function $\mathcal{L}_{CLIP}$, a second-order Taylor expansion around the parameters $\vtheta$ yields:
\begin{equation}
  \label{eq:lossdelta}
  \mathcal{L}_{CLIP}(\vtheta + \vdelta) \simeq
  \mathcal{L}_{CLIP}(\vtheta) + \nabla_{\vtheta} \mathcal{L}_{CLIP}(\vtheta)^\top \vdelta + \frac{1}{2} \vdelta^{\top} \mF(\vtheta) \vdelta, 
\end{equation} 
where $\vdelta$ represents the parameter update direction, and the curvature of the loss landscape is characterized by the Fisher information matrix $\mF(\vtheta)$, defined as \cite{pascanu2013revisiting}:
\begin{equation}
  \label{eq:fisher}
  \mF(\vtheta) = \E[ \nabla_{\vtheta} \mathcal{L}_{CLIP} \cdot \nabla_{\vtheta} \mathcal{L}_{CLIP}^{\top} ].
\end{equation}
Minimizing this local quadratic approximation leads to the natural gradient update: $\vdelta^{\ast} = -\mF(\vtheta)^{-1} \nabla_{\vtheta} \mathcal{L}_{CLIP}$.

\subsection{Bayesian Approach}

The Bayesian framework treats the parameters $\vtheta$ as random variables and seeks to infer their posterior distribution given the dataset $\D$ using Bayes’ rule \cite{murphy2023probabilistic, murphy2012machine}: 
\begin{equation}
  \label{eq:bayes}
  p(\vtheta \mid \D) = \frac{p(\D \mid \vtheta) p(\vtheta)}{p(\D)} = \frac{p(\D \mid \vtheta) p(\vtheta)}{ \int p(\D \mid \vtheta') p(\vtheta') d\vtheta' },
\end{equation}
where $p(\D \mid \vtheta)$ is the likelihood, $p(\vtheta)$ represents the prior distribution over parameters, and $p(\D)$ denotes the evidence (or marginal likelihood).
Once the posterior $p(\vtheta \mid \D)$ is obtained, predictions for a new input can be made by: 
\begin{equation}
  \label{eq:bayes_prediction}
  p(\vy' \mid \vx', \D) = \int p(\vy' \mid \vx', \vtheta) \, p(\vtheta \mid \D) d\vtheta.
\end{equation}
This process, known as exact Bayesian inference, is generally infeasible for high-dimensional models due to the intractability of computing integrals over the full parameter space. As a result, approximate inference methods are typically used to estimate the posterior and make Bayesian inference computationally practical.

\section{Methodology}
\label{sec:Methodology}

\subsection{Problem Setting}
\label{sec:Problem Setting}
We consider a CLIP model pre-trained on a source ID dataset $\D_s$, where the model is reparameterized as $\vtheta = \text{vec}(\vtheta^{\vi} ,\vtheta^{\vt}) \in \R^n$, while keeping the backbone frozen during subsequent adaptation (Figure \ref{fig: Adapter}).
%
%
The goal is fine-tuning the trainable parameter $\vtheta$ using a target dataset $\D_t = \{\vx_j=(\vx^{\vi}_j, \vx^{\vt}_j) \sim p_{t}(\vx) \}_{j=1}^{N_t}$, where $p_{t}(\vx) = p_{s}(\vx)$ for ID and $p_{t}(\vx) \neq p_{s}(\vx)$ for the OOD target dataset. We also consider $\B_k= \{(\vx^{\vi}_j, \vx^{\vt}_j) \}_{j=1}^{m} \subseteq \D_t$ as a minibatch of pairs of $(\vx^{\vi}_j, \vx^{\vt}_j)$ at the step of $k$ with the size of $m$.

\begin{figure*}[t]
\centerline{\includegraphics[width=\textwidth]{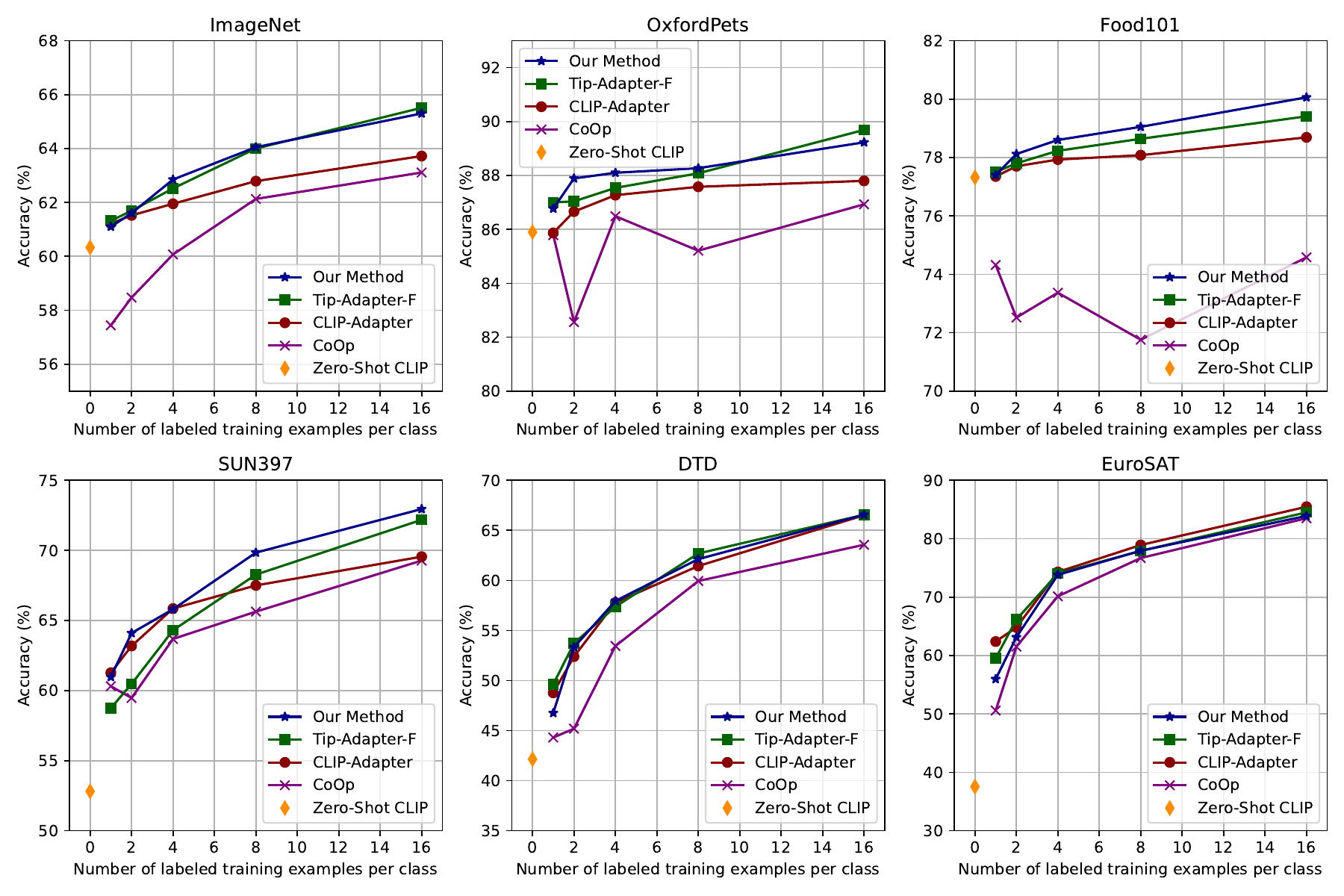}}
\caption{Accuracy results of different few-shot fine-tuning scenarios for six image classification datasets. Our method (blue) consistently achieves superior ID performance in every few-shot setup, and in certain cases, performs comparably to the baselines: Tip-Adapter-F (green), CLIP-Adapter (red), CoOp (purple), and Zero-Shot CLIP (orange).}
\label{fig: main result}
\end{figure*}

\subsection{Algorithm}
\label{sec:Algorithm}

In line with the standard Kalman filtering approach, we begin by specifying a Gaussian initial prior distribution over the trainable parameters: $p(\vtheta_0) = \gN(\vmu_0,\mSigma_0)$. Our goal is to recursively estimate the posterior distribution of the parameters through Bayesian inference \citep{murphy2023probabilistic, murphy2012machine}:
\begin{equation} \label{Eq. posterior}
    p(\vtheta_k \mid \B_{1:k}) \propto p(\vy_k \mid \B_k , \vtheta_{k-1}) p(\vtheta_{k-1} \mid \B_{1:k-1}),
\end{equation}
where $p(\vtheta_{k-1} \mid \B_{1:k-1})$ represents the prior (i.e., the posterior from the previous step), and $p(\vy_k \mid \B_k , \vtheta_{k-1})$ denotes the likelihood of output $\vy_k$ given the minibatch $\B_k$ and trainable parameters $\vtheta_{k-1}$. This posterior update can be computed recursively for each minibatch $\B_k$ using Kalman filtering.
Consider a Gaussian likelihood function as $p(\vy_k \mid \B_k , \vtheta_{k-1}) = \gN(\vy_k \mid \hat{\vy}_k, \mR_k)$,
%
%
%
where $\mR_k = \verb|Cov|(\vy_k|\hat{\vy}_k)$ represents the covariance matrix of the observation noise. Here, $\vy_k$ denotes the true output, while $\hat{\vy}_k$ corresponds to the estimated output of the model. The estimated output $\hat{\vy}_k$ can be defined as the following equation:
\begin{equation} \label{Eq. model}
    \hat{\vy}_k = h(\B_{k}, \vtheta_k) = \verb|diag| \left(\mS_C(\vI,\vT) \right).
\end{equation}
%
%
Here, the function $\verb|diag|(\cdot)$ extracts the diagonal entries of the matrix $\mS_C$.
To maximize the diagonal entries of the cosine similarity matrix $\mS_C$, the true output of the model, $\vy_k$, is defined as follows:
\begin{equation} \label{Eq. model}
    \vy_k = \verb|diag|(\mI_{m \times m}).
\end{equation}
%
%
Here, $\mI_{m \times m}$ represents the identity matrix of size $m \times m$.
We also evaluated the alternative formulation $\vy_k = \verb|flatten|(\mI_{m \times m})$, where the estimated output is given by $\hat{\vy}_k = h(\B_{k}, \vtheta_k) = \verb|flatten| \left(\mS_C(\vI,\vT) \right)$. Here, the function $\verb|flatten|(\cdot)$ reshapes the matrix $\mS_C$ and $\mI$ of size $m \times m$ into a vector of size $1 \times (m \times m)$. However, empirical results indicate that the approach employing the $\verb|diag|(\cdot)$ function yields superior performance. Consequently, we omit further consideration of the $\verb|flatten|(\cdot)$-based approach.
%

Furthermore, it is worth noting that a more general form of the likelihood function involves a non-Gaussian function from the exponential family, expressed as $p(\vy_k \mid \B_k , \vtheta_{k-1}) = \verb|exp|(T(\vy_k)|\hat{\vy}_k)$, where $\verb|exp|(\cdot)$ denotes the exponential family function, and $T(\cdot)$ represents the sufficient statistics associated with the exponential family. In this context, the observation noise covariance matrix $\mR_k$ is defined as $\mR_k = \verb|Cov|(T(\vy_k)|\hat{\vy}_k)$, which captures the covariance between $T(\vy_k)$ and $\hat{\vy}_k$ \citep{ollivier2018online}.

To estimate the posterior distribution $p(\vtheta_k \mid \B_{1:k}) = \gN(\vmu_k,\mSigma_k)$ from the prior distribution $p(\vtheta_{k-1} \mid \B_{1:k-1}) = \gN(\vmu_{k-1},\mSigma_{k-1})$ based on Bayesian inference, we follow the Kalman algorithm as described below:

\paragraph{Prediction Step} In accordance with the first Kalman step, the prior is predicted based on the posterior of the previous time step using the following process model:
\begin{equation} \label{Eq. prior}
    p(\vtheta_{k \mid k-1} \mid \vtheta_{k-1}) = \gN(\vtheta_{k \mid k-1} \mid \vtheta_{k-1}, \mQ_{k}) ,
\end{equation}
where we set $\mQ_{k} = q\mI$ with $q \geq 0$.
The prediction step is given by:
\begin{subequations} \label{Eq. LoKA Prediction}
        \begin{align} 
            \vmu_{k|k-1} &= \vmu_{k-1}
            \label{Eq. LoKO Prediction State}\\
            \mSigma_{k|k-1} &= \mSigma_{k-1} + \mQ_{k}
            \label{Eq. LoKA Prediction Covariance}
        \end{align}
\end{subequations}
\paragraph{Updating Step} In the second step of the Kalman filter, the predicted prior is updated to estimate the posterior as follows:
\begin{subequations} \label{Eq. Algorithm Updating}
    \begin{align} 
        \mK_k &= \mSigma_{k|k-1} \mH_k^{\top} \left(\mH_k \mSigma_{k|k-1} \mH_k^{\top} + \mR_k\right)^{-1}
        \label{Eq. Algorithm diag gain update}\\
        \vmu_{k} &= \vmu_{k|k-1} + \mK_k \left(\vy_k - h(\B_{k}, \vmu_{k|k-1}) \right)
        \label{Eq. Algorithm Updating Process}\\
        \mSigma_{k} &= \mSigma_{k|k-1} - \mK_k \mH_k \mSigma_{k|k-1}
        \label{Eq. Algorithm diag cov update}
    \end{align}
\end{subequations}
%
%
Here, $\mK_k$ is the Kalman gain, and $\mH_k$ indicates the Jacobian matrix of the function $h(\B, \vtheta)$ with respect to the parameters $\vtheta$ at the point of $(\B_{k}, \vmu_{k|k-1})$.

\subsection{Kalman as NGD}
\label{sec: Kalman as NGD}

We now demonstrate that the Kalman algorithm equivalently updates the trainable parameters along the natural gradient direction.
When minimizing $\mathcal{L}_{\text{CLIP}}$, the CLIP model seeks parameter values $\vtheta$ such that the predictive likelihood $p(\vy|\B, \vtheta)$ closely approximates the target conditional data distribution $p_t(\vy|\B)$. This corresponds to minimizing the Kullback–Leibler divergence $D_{KL}(p_t(\vy|\B)||p(\vy|\B, \vtheta))$.
Given the Gaussian likelihood, it can be equivalent to minimizing: $\mathcal{L}_{\text{CLIP}} \equiv \frac{1}{2}(\hat{\vy}_k-\vy_k)^{\top} \mR_{k}^{-1} (\hat{\vy}_k-\vy_k)$ \cite{goodfellow2016deep, murphy2012machine, bishop2006pattern}. Now, we can define the Lemma \ref{lem:modified_kamlan_update}.


%
\begin{lemma}
    \label{lem:modified_kamlan_update}
    The update step of the standard Kalman algorithm, as presented in Equation \eqref{Eq. Algorithm Updating}, can be reformulated as: 
    \begin{subequations} 
    \begin{align} 
        \vmu_{k} &= \vmu_{k|k-1} - \mSigma_{k} \nabla_{\vtheta} \mathcal{L}_{\text{CLIP}}
        \label{eq:posterior_mean_update} \\
        \mSigma_{k}^{-1} &= \mSigma_{k|k-1}^{-1} + \mH_k^{\top} \mR_{k}^{-1} \mH_k \label{eq:inv_cov_update} 
    \end{align}
    \end{subequations}
\end{lemma}

%
%

\begin{proof}
    To derive Equation \eqref{eq:posterior_mean_update}, we begin by computing the gradient of the loss with respect to the parameters using the chain rule: $\nabla_{\vtheta} \mathcal{L}_{\text{CLIP}} = \nabla_{\vtheta}^{\top} \hat{\vy} \nabla_{\hat{\vy}} \mathcal{L}_{\text{CLIP}}$. We know that $\nabla_{\vtheta} \hat{\vy} = \mH_k$, and $\nabla_{\hat{\vy}} \mathcal{L}_{\text{CLIP}} = \mR_k^{-1} (\hat{\vy}_k-\vy_k)$. Therefore, the gradient simplifies to $\nabla_{\vtheta} \mathcal{L}_{\text{CLIP}} = \mH_k^{\top} \mR_k^{-1} (\hat{\vy}_k-\vy_k)$. By rearranging terms, we obtain the prediction error as $ (\vy_k-\hat{\vy}_k) = - \mR_k {\mH_k^{\top}}^{-1} \nabla_{\vtheta} \mathcal{L}_{\text{CLIP}}$. Substituting this expression into the Kalman parameter update from Equation \eqref{Eq. Algorithm Updating Process}, we get: $ \vmu_{k} = \vmu_{k|k-1} - \mK_k \mR_k {\mH_k^{\top}}^{-1} \nabla_{\vtheta} \mathcal{L}_{\text{CLIP}}$.
    Next, we aim to simplify the term $\mK_k \mR_k$ by expressing it as $\mSigma_k\mH_k^{\top}$. To do this, we use the identity $\mK_k \mR_k = \mK_k(\mR_k +\mH_k \mSigma_{k|k-1} \mH_k^{\top}) - \mK_k \mH_k \mSigma_{k|k-1} \mH_k^{\top}$, and by applying Equation \eqref{Eq. Algorithm diag gain update}, we know that: $\mK_k(\mR_k +\mH_k \mSigma_{k|k-1} \mH_k^{\top}) = \mSigma_{k|k-1} \mH_k^{\top} $.
    So: $\mK_k \mR_k = \mSigma_{k|k-1} \mH_k^{\top} - \mK_k \mH_k \mSigma_{k|k-1} \mH_k^{\top}$, which simplifies to $\mK_k \mR_k = \big( \mSigma_{k|k-1} - \mK_k \mH_k \mSigma_{k|k-1} \big) \mH_k^{\top}$.
    Using the covariance update formula from Equation \eqref{Eq. Algorithm diag cov update}, we know that $\big( \mSigma_{k|k-1} - \mK_k \mH_k \mSigma_{k|k-1} \big) = \mSigma_{k}$, and therefore: $\mK_k \mR_k = \mSigma_{k} \mH_k^{\top}$. Substituting this back into the expression for $\vmu_k$, we obtain: $\vmu_{k} = \vmu_{k|k-1} - \mSigma_{k} \mH_k^{\top} {\mH_k^{\top}}^{-1} \nabla_{\vtheta} \mathcal{L}_{\text{CLIP}}$. 
    This simplifies further to $\vmu_{k} = \vmu_{k|k-1} - \mSigma_{k} \nabla_{\vtheta} \mathcal{L}_{\text{CLIP}}$, which confirms Equation \eqref{eq:posterior_mean_update}.

    To derive Equation \eqref{eq:inv_cov_update}, we start by substituting the Kalman gain from Equation \eqref{Eq. Algorithm diag gain update} into the covariance update expression in Equation \eqref{Eq. Algorithm diag cov update}, which gives: $\mSigma_{k} = \mSigma_{k|k-1} - \mSigma_{k|k-1} \mH_k^{\top} \left(\mH_k \mSigma_{k|k-1} \mH_k^{\top} + \mR_k\right)^{-1} \mH_k \mSigma_{k|k-1}$.
    To rewrite this in inverse form, we use the Woodbury matrix identity, which states that: $(\mA + \mU\mC\mV)^{-1} = \mA^{-1}-\mA^{-1}\mU(\mC^{-1}+\mV\mA^{-1}\mU)^{-1}\mV\mA^{-1}$. By identifying the terms as $\mA = \mSigma^{-1}_{k|k-1}$, $\mU = \mH_k^{\top}$, $\mC = \mR_k^{-1}$, and $\mV = \mH_k$, we apply the identity to obtain the inverse form of the covariance update as: $\mSigma^{-1}_{k} = \mSigma_{k|k-1}^{-1} + \mH_k^{\top} \mR_{k}^{-1} \mH_k$. 
\end{proof}
This lemma plays an important role in demonstrating the equivalence between the Kalman update and the natural gradient direction, as shown in Proposition \ref{theorem:equivalent_fisher_for_kalman}.
\begin{proposition}
    \label{theorem:equivalent_fisher_for_kalman}
    Assuming a Gaussian (or more broadly, exponential family) likelihood, the Kalman update step in Equation \eqref{Eq. Algorithm Updating Process} provides a close approximation to the natural gradient direction, given by $\vdelta^{\ast} = -\mF(\vtheta_k)^{-1} \nabla_{\vtheta} \mathcal{L}_{\text{CLIP}_k}$.
\end{proposition}
\begin{proof}
    We start with the Fisher information matrix as defined in Equation \eqref{eq:fisher}, which can be expressed as: $\mF(\vtheta) = \E[ \nabla_{\vtheta} \mathcal{L}_{\text{CLIP}} \cdot \nabla_{\vtheta} \mathcal{L}_{\text{CLIP}}^{\top} ]$. At time step $k$, the contribution from the current minibatch to the Fisher information matrix is given by $\mF(\vtheta_k) = \nabla_{\vtheta} \mathcal{L}_{\text{CLIP}_k} \cdot \nabla_{\vtheta} \mathcal{L}_{\text{CLIP}_k}^{\top}$. Applying the chain rule yields $\nabla_{\vtheta}\mathcal{L}_{\text{CLIP}_k} = \nabla_{\hat{\vy}}\mathcal{L}_{\text{CLIP}_k} \cdot \nabla_{\vtheta}\hat{\vy}_k$, and thus: $\mF(\vtheta_k) = (\nabla_{\hat{\vy}}\mathcal{L}_{\text{CLIP}_k} \cdot \nabla_{\vtheta}\hat{\vy}_k) \cdot (\nabla_{\hat{\vy}}\mathcal{L}_{\text{CLIP}_k} \cdot \nabla_{\vtheta}\hat{\vy}_k)^{\top}$.
    Using matrix multiplication properties, this can be rewritten as: $\mF(\vtheta_k) = \nabla_{\vtheta}^{\top}\hat{\vy}_k \cdot \nabla_{\hat{\vy}}^2 \mathcal{L}_{\text{CLIP}_k} \cdot \nabla_{\vtheta}\hat{\vy}_k$. Noting that $\mH_k = \nabla_{\vtheta}\hat{\vy}_k$, and assuming a Gaussian (or more generally, exponential family) likelihood, it follows that $\mR_k^{-1} = \nabla_{\hat{\vy}}^2 \mathcal{L}_{\text{CLIP}_k}$. Consequently, the Fisher information matrix can be expressed as $\mF(\vtheta_k) = \mH_k^{\top} \mR_k^{-1} \mH_k$, which corresponds to the update term in the inverse covariance formulation of Equation \eqref{eq:inv_cov_update} in Lemma \ref{lem:modified_kamlan_update}, i.e. $\mF \equiv \mSigma^{-1}$. Therefore, the natural gradient step $\mF(\vtheta_k)^{-1} \nabla_{\vtheta} \mathcal{L}_{\text{CLIP}_k}$ is equivalent to $ \mSigma_{k} \nabla_{\vtheta} \mathcal{L}_{\text{CLIP}_k} $.
\end{proof}

This proposition provides a theoretical foundation showing that the Kalman algorithm’s update step aligns with the natural gradient direction, which allows the Kalman filter to perform optimization that resembles the natural gradient descent.

\begin{table*}[t]
  \caption{Accuracy results for the ImageNet, OxfordPets, Food101, SUN397, DTD, and EuroSAT datasets. The table compares performance across 1, 2, 4, 8, and 16-shot fine-tuning scenarios achieved by our method and baseline approaches: Tip-Adapter-F, CLIP-Adapter, CoOp, and Zero-Shot CLIP. Bolded values indicate the best performance in each column.}
  \label{table: main result}
  \centering
  {\small
  \begin{tabular}{lc|c|c|c|c|c|c}
    \toprule \addlinespace[1ex]
    \multicolumn{2}{c|}{\textbf{Method}} & \textbf{ImageNet}  &  \textbf{OxfordPets} & \textbf{Food101} & \textbf{SUN397} & \textbf{DTD} & \textbf{EuroSAT}  \\ \addlinespace[1ex]
    \toprule
    \textbf{Zero-Shot CLIP} &  & \multirow{2}{*}{$60.33$} & \multirow{2}{*}{$85.90$} & \multirow{2}{*}{$77.32$} & \multirow{2}{*}{$52.81$} & \multirow{2}{*}{$42.15$} & \multirow{2}{*}{$37.55$} \\
    \citep{radford2021learning} & & & & & & & \\ \midrule \addlinespace[1ex] 
    \multirow{5}{*}{\textbf{CoOp}} & 1-Shot  & $57.44$ & $85.79$ & $74.32$ & $60.31$ & $44.31$ & $50.57$ \\
    & 2-Shot  & $58.47$ & $82.56$ & $72.52$ & $59.47$ & $45.19$ & $61.52$ \\
    & 4-Shot  & $60.07$ & $86.49$ & $73.37$ & $63.68$ & $53.45$ & $70.14$ \\
    \citep{zhou2022learning} & 8-Shot  & $62.13$ & $85.21$ & $71.76$ & $65.63$ & $59.94$ & $76.69$ \\ 
    & 16-Shot  & $63.11$ & $86.93$ & $74.58$ & $69.27$ & $63.56$ & $83.50$ \\ \midrule \addlinespace[1ex]
    \multirow{5}{*}{\textbf{CLIP-Adapter}} & 1-Shot  & $61.21$ & $85.87$ & $77.35$ & $61.27$ & $48.77$ & $62.39$ \\
    & 2-Shot  & $61.52$ & $86.66$ & $77.70$ & $63.18$ & $52.39$ & $64.81$ \\
    & 4-Shot  & $61.95$ & $87.27$ & $77.93$ & $65.85$ & $57.79$ & $74.30$ \\
    \citep{gao2024clip} & 8-Shot  & $62.79$ & $87.58$ & $78.08$ & $67.50$ & $61.44$ & $78.92$ \\ 
    & 16-Shot  & $63.72$ & $87.80$ & $78.69$ & $69.55$ & $66.51$ & $\textbf{85.47}$ \\ \midrule \addlinespace[1ex]
    \multirow{5}{*}{\textbf{Tip-Adapter-F}} & 1-Shot  & $61.32$ & $87.00$ & $77.50$ & $58.75$ & $49.63$ & $59.53$ \\
    & 2-Shot  & $61.69$ & $87.04$ & $77.80$ & $60.47$ & $53.72$ & $66.17$ \\
    & 4-Shot  & $62.52$ & $87.54$ & $78.23$ & $64.30$ & $57.38$ & $74.04$ \\
    \citep{zhang2022tipadaptertrainingfreeadaptionclip} & 8-Shot  & $64.00$ & $88.08$ & $78.64$ & $68.27$ & $62.69$ & $77.91$ \\ 
    & 16-Shot  & $\textbf{65.51}$ & $\textbf{89.69}$ & $79.41$ & $72.18$ & $66.56$ & $84.51$ \\ \midrule \addlinespace[1ex]
    \multirow{5}{*}{\colorbox{pink}{\textbf{Our Method}}}     & 1-Shot  & $61.10$ & $86.77$ & $77.40$ & $60.97$ & $46.76$ & $55.97$ \\
    & 2-Shot  & $61.60$ & $87.89$ & $78.12$ & $64.10$ & $53.32$ & $63.13$ \\
    & 4-Shot  & $62.85$ & $88.10$ & $78.60$ & $65.80$ & $57.95$ & $73.81$ \\
    & 8-Shot  & $64.05$ & $88.27$ & $79.05$ & $69.84$ & $62.13$ & $77.94$ \\ 
    & 16-Shot  & $65.30$ & $89.23$ & $\textbf{80.06}$ & $\textbf{72.95}$ & $\textbf{66.56}$ & $83.89$ \\ \addlinespace[1ex]

    \bottomrule
  \end{tabular}
  }
\end{table*}

\subsection{Robustness to OOD}
\label{sec:Robustness to OOD}

Inspired by \citep{altamirano2023robust, duran2024outlier}, we dynamically adjust the Kalman filtering update step to further enhance its robustness to OOD data during fine-tuning. To achieve this, we propose two novel methods for estimating the observation noise covariance matrix, $\mR_k$.
To empirically estimate the $\mR_k$, we employ an exponential moving average (EMA) approach that aggregates statistical information from previously observed data: 
\begin{equation} \label{Eq. EMA matrix R}
     \mR_k = \beta \mR_{k-1} + (1-\beta) \hat{\mR}_k,
\end{equation}
where $\beta \in (0,1)$ is the forgetting factor, and $\hat{\mR}_k$ denotes the contribution from the current minibatch, which can be calculated using Taylor series expansion around the point $\vtheta_k \approx \vmu_{k|k-1}$ via:
\paragraph{Method 1} 
The first method relies on the zeroth-order Taylor series expansion. In this method, $\hat{\vy}_k$ is approximated as $h(\B_{k}, \vmu_{k \mid k-1})$, which leads to the following formulation:
\begin{subequations} \label{Eq. matrix R approach1}
\begin{align}
    \hat{\mR}_k &= \left(\vy_k - h(\B_{k}, \vmu_{k \mid k-1}) \right)\left(\vy_k - h(\B_{k}, \vmu_{k \mid k-1}) \right)^{\top}. 
    \label{Eq. matrix R approach1b}
\end{align}
\end{subequations}
\paragraph{Method 2} 
This method uses the first-order Taylor series expansion. In this method, $\hat{\vy}_k$ is approximated as $h(\B_{k}, \vmu_{k \mid k-1}) + \mH_k (\vtheta_{k} - \vmu_{k|k-1})$. The corresponding formulation is expressed as follows:
\begin{subequations} \label{Eq. matrix R approach2}
\begin{align}
    \begin{split}
        \hat{\mR}_k &= \left(\vy_k - h(\B_{k}, \vmu_{k \mid k-1}) \right)\left(\vy_k - h(\B_{k}, \vmu_{k \mid k-1}) \right)^{\top}\\
        & + \mH_k \mSigma_{k|k-1} \mH_k^{\top}.
    \end{split}
    \label{Eq. matrix R approach2b}
\end{align}
\end{subequations}
Note that this method will not add extra computational cost since the operation of $\mH_k \mSigma_{k|k-1} \mH_k^{\top}$ will be part of the Kalman gain calculation in~\eqref{Eq. Algorithm diag gain update}. \\

Now, we need to evaluate the distance of the observed minibatch relative to the training data distribution. To this aim, the Mahalanobis distance is calculated using the observation noise covariance matrix estimated in the previous time step, $\mR_{k-1}$:
\begin{equation} \label{Eq. Mahalanobis distance}
    d_M = \sqrt {\left(\vy_k - \hat{\vy}_k\right) \boldsymbol{R}_{k-1}^{-1} \left(\vy_k - \hat{\vy}_k\right)^{\top}}. 
\end{equation}
Here, the Mahalanobis distance, $d_M$, quantifies the deviation of the observed minibatch from the distribution of the training set, which is characterized by $\mR_{k-1}$. Then, we can calculate the regulation term as follows:
\begin{equation} \label{Eq. regulation term}
     \lambda = \mathbf{e}^{-\alpha d_M},
\end{equation}
where $\alpha$ is the scaling factor, treated as a hyperparameter, that modulates the influence of the Mahalanobis distance $d_M$. If the minibatch deviates significantly from the training set distribution, the regulation term diminishes and mitigates the effect of the minibatch on the updates. In contrast, when the minibatch aligns closely with the distribution of the training set, the regulation term approaches $1$, and effectively restores the behavior of the original algorithm. The regulation term $\lambda$ is multiplied by the update terms in Equations \eqref{Eq. Algorithm Updating Process},~\eqref{Eq. Algorithm diag cov update}, and~\eqref{Eq. EMA matrix R} to update $\vmu$, $\mSigma$, and $\mR$, respectively. A pseudocode of our final approach is shown in Algorithm \ref{alg:algorithm1}.

\begin{table*}
  \caption{OOD test set accuracy for distribution-shifted versions of the ImageNet dataset: ImageNetV2, ImageNet-Sketch, ImageNet-A, and ImageNet-R. The final column presents the average accuracy across all datasets. Bolded values indicate the highest performance in each column.}
  \label{table: distribution shift}
  \centering
  {\small
  \begin{tabular}{l|c|c|c|c|c|c}
    \toprule \addlinespace[1ex]
     \multicolumn{1}{c|}{\textbf{Method}} & \textbf{ImageNet} & \textbf{ImageNetV2} & \textbf{ImageNet-Sketch} &  \textbf{ImageNet-A} &  \textbf{ImageNet-R} & \textbf{Ave.} \\
     \midrule
     \textbf{Zero-Shot CLIP} \citep{radford2021learning} & 60.33 & 53.62 & 34.38 & 21.64 & 50.00 & 43.99\\
     \textbf{CoOp} \citep{zhou2022learning} & 63.11 & 54.84 & 32.86 &  \textbf{22.14} & 54.94 & 45.58\\
     \textbf{CLIP-Adapter} \citep{gao2024clip} & 63.72 & 55.46 & 36.87 & 20.93 & 58.94 & 47.18\\
     \textbf{Tip-Adapter-F} \citep{zhang2022tipadaptertrainingfreeadaptionclip} & \textbf{65.51} & 57.10 & 36.18 & 20.84 & 60.33 & 47.99\\
     \colorbox{pink}{\textbf{Our Method}} & 65.30 & \textbf{58.84} & \textbf{37.38} & 21.27 & \textbf{61.90} & \textbf{48.93}\\
     
    \bottomrule
  \end{tabular}
  }
\end{table*}

\section{Experiments and Analysis}
\label{sec:Experiments and Analysis}
\subsection{Setup}
\label{sec:Setup}
This section presents the experimental results across diverse ID and OOD scenarios, utilizing various widely-used image classification datasets. In particular, our experimental setup is as follows:
\paragraph{Datasets}
We conduct ID experiments on various few-shot fine-tuning scenarios using ImageNet \citep{vinyals2016matching}, OxfordPets \citep{parkhi2012cats}, Food101 \citep{bossard2014food}, SUN397 \citep{xiao2010sun}, DTD \citep{cimpoi2014describing}, and EuroSAT \citep{helber2019eurosat}, comparing the results with multiple baselines. Furthermore, we conduct 16-shot fine-tuning experiments on well-established image classification datasets such as MNIST \citep{lecun1998gradient}, FashionMNIST \citep{xiao2017fashion}, CIFAR-10/100 \citep{krizhevsky2009learning}, and Places365 \citep{zhou2017places}, to evaluate improvements over zero-shot learning. Furthermore, we performed OOD experiments on ImageNet distribution-shifted variants, including ImageNetV2 \citep{recht2019imagenet}, ImageNet-Sketch \citep{wang2019learning}, ImageNet-A \citep{hendrycks2021natural}, ImageNet-R \citep{hendrycks2021many}, and ImageNet-C \citep{hendrycks2019robustness}.
\paragraph{Baselines}
We compare the performance of our algorithm with four baseline models: Zero-Shot CLIP \citep{radford2021learning}, CoOp \citep{zhou2022learning}, CLIP-Adapter \citep{gao2024clip}, and Tip-Adapter-F \citep{zhang2022tipadaptertrainingfreeadaptionclip}. We use the CLIP pre-trained model for the zero-shot scenario and other methods for few-shot fine-tuning scenarios. For a fair comparison, we select the best variant of CoOp to place the class token (at the end of the 16-token prompts).
\paragraph{Training Settings}
Our few-shot experiments are conducted using training sets with 1, 2, 4, 8, and 16 images per class, and the models are evaluated on the full test set. 
Both image encoder and text encoder leverage pre-trained weights provided by \citep{radford2021learning}, which serve as the backbones of our algorithm. We use the low-rank decomposition technique \citep{hu2021lora} for parameterization of the adapter, and freeze the backbones. We set the scaling factor $\alpha$ to 0.1 and the forgetting factor $\beta$ to 0.98.
To implement the Kalman filter efficiently, we adopt the method proposed by \cite{chang2022diagonal}, which reduces the computational complexity by one order of magnitude with respect to the number of trainable parameters ($n$).
For non-learnable methods, we use the handcrafted prompt "\textit{a photo of a} [CLASS]" for all datasets, except for EuroSAT, where we use the prompt "\textit{a centered satellite photo of a} [CLASS]". For image pre-processing, we follow the protocol in \citep{radford2021learning}, which involves resizing, center cropping, and normalizing pixel values. For baselines, we employ the AdamW optimizer with a linear learning rate decay schedule (weight decay = 0.0001), starting at $0.001$.
In our algorithm, the batch size is set to 10, with 2 epochs for smaller datasets such as OxfordPets and Food101, and 5 epochs for larger datasets such as ImageNet and SUN397. All experiments were conducted on an NVIDIA GeForce RTX 4090 GPU platform.
\subsection{Main Results}
\label{sec:Main Results}
\paragraph{ID experiments}
For the ID scenario, we used pre-trained backbones on ImageNet as the source dataset. The pre-trained model was then fine-tuned on ID target datasets across various few-shot scenarios. We report the test set accuracy in Figure~\ref{fig: main result} and Table~\ref{table: main result}, and compare our algorithm to baseline methods. As demonstrated, our algorithm consistently achieves superior ID performance in every few-shot scenario and, in certain cases, performs comparably to the baselines.
Specifically, on datasets such as OxfordPets, Food101, and SUN397, our algorithm demonstrates notable performance gains, especially as the number of labeled examples increases. Although the performance differences in ImageNet are less significant, our algorithm remains competitive and also demonstrates clear improvements on DTD and EuroSAT. The numerical results in Table 1 support the trends shown in Figure 2, with our algorithm frequently achieving the highest accuracy in various few-shot setups. These results indicate that our Kalman-based fine-tuning approach enables NGD-based adaptation of CLIP-based models to new tasks.

Furthermore, in Figure~\ref{fig: bar chart}, we provide additional experiments on some well-established image classification datasets. The figure highlights the improvement in test set accuracy achieved by our method after a 16-shot fine-tuning compared to zero-shot CLIP results.
As seen in Figure ~\ref{fig: bar chart}, the proposed method shows significant improvements in test set accuracy after 16 shots of fine-tuning compared to CLIP results from zero shots, with improvements ranging from about $2.7\%$ to $46\%$ depending on the dataset. The largest improvement is on the EuroSAT dataset, while smaller improvements are shown on the OxfordPets and Food101 datasets.
\paragraph{OOD experiments}
To evaluate the performance of our algorithm in OOD scenarios, we design and conduct two distinct experiments. In the first experiment, we assess the OOD generalization of our algorithm after training to show the inherent ability of uncertainty quantification in our Bayesian approach. Specifically, the model is trained using ImageNet as the ID dataset and evaluated on various distribution-shifted versions of ImageNet, which serve as the OOD target datasets. The test accuracy, compared to the baseline methods, is presented in Table~\ref{table: distribution shift}.
\begin{figure}
  \centering
    \includegraphics[width=1.0\linewidth]{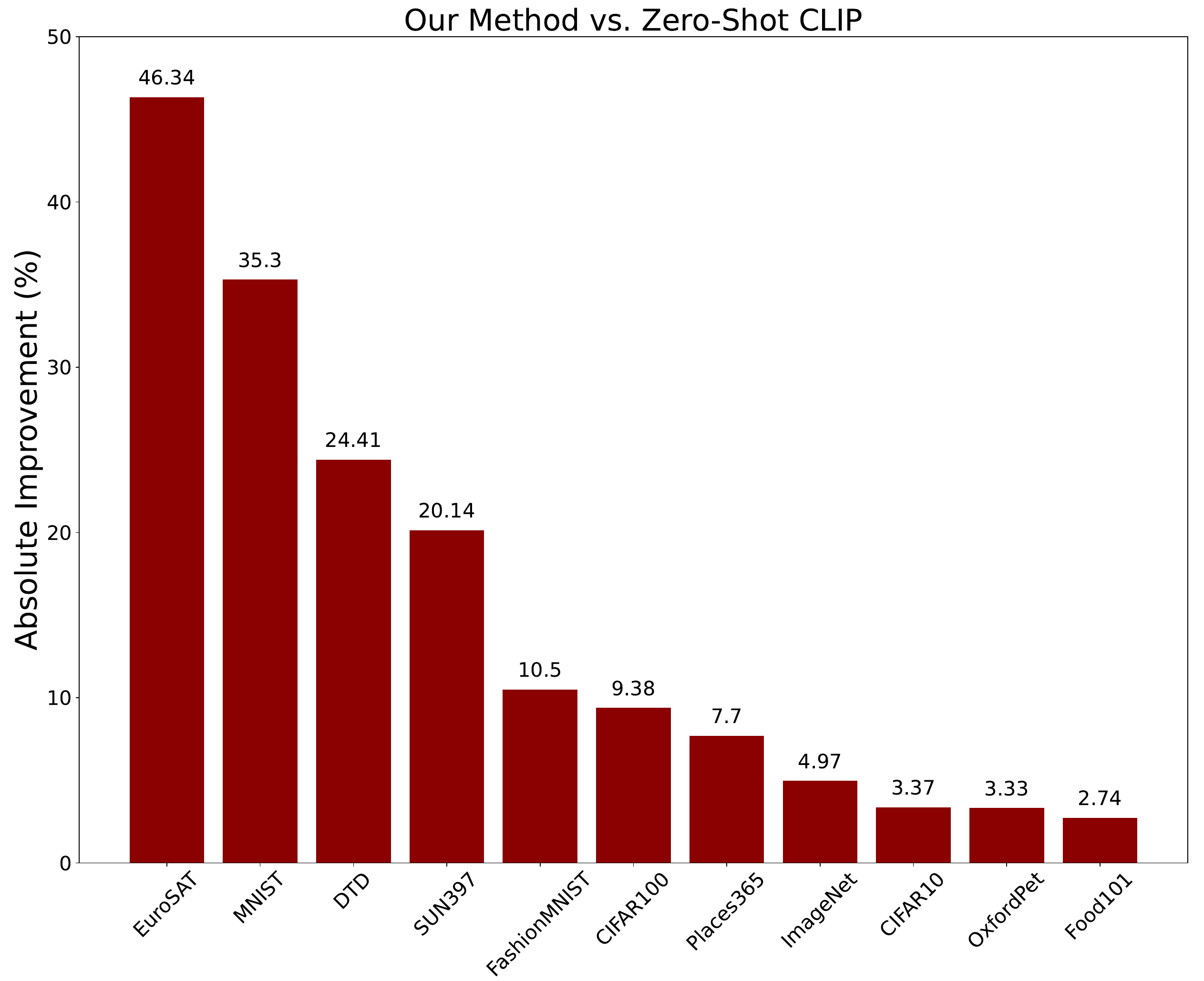}
  \caption{Absolute improvement in accuracy across 11 image classification datasets achieved by our method compared to Zero-Shot CLIP. Results are reported for the 16-shot fine-tuning scenario.}
  \label{fig: bar chart}
\end{figure}
Our algorithm achieves the highest average accuracy of $48.93\%$ in the distribution-shifted data sets, compared to $47.99\%$ of Tip-Adapter-F, $47.18\%$ of CLIP-Adapter, $45.58\%$ of CoOp and $43.99\%$ of Zero-Shot CLIP.

In the second experiment, we examine the robustness of our algorithm during training. For this purpose, the model is trained on ImageNet as the ID dataset, with corrupted images from ImageNet-C introduced as OOD data during fine-tuning. The experiments are carried out using five different values of the scaling factor, $\alpha$, and five varying proportions of the OOD data incorporated into the fine-tuning process. The corresponding test accuracy results are reported in Table~\ref{Table: alpha vs outlier}.
For higher proportions of OOD data during fine-tuning, a more aggressive regulation of the update step, and higher values of scaling factors ($\alpha$) tend to produce more robust performance compared to the original Kalman algorithm ($\alpha = 0$). 
Specifically, the results show that the use of $\alpha = 1$ can increase the accuracy by almost $9\%$ compared to the original Kalman filter. On the other hand, when the proportion of OOD data is low, a more conservative adjustment to the Kalman update and lower values of the scaling factor (such as $\alpha = 0.1$) perform better. 
Furthermore, the results show that when there are only small amounts of OOD data in the training set, the use of aggressive regulation with higher $\alpha$ values can be harmful and can cause instability during training.

\begin{table}
  \caption{Effect of scaling factor ($\alpha$) on model robustness to varying OOD data percentages from ImageNet-C in the ImageNet training dataset.}
  \label{Table: alpha vs outlier}
  \centering
  {\small
  \begin{tabular}{|c|c|c|c|c|c|}
    \toprule 
     \textbf{Scaling} & \multicolumn{5}{c|}{\textbf{Percentage of OOD Data}} \\ \addlinespace[0.5ex]
     \textbf{Factor ($\alpha$)} & 1\% & 5\% & 10\% &  25\% &  50\% \\ \addlinespace[0.5ex]
     \toprule 
     $ 0.00$ & 64.90 & 63.45 & 59.33 & 48.03 & 46.87\\
     $0.01$ & 64.59 & 62.72 & 60.62 & 49.73 & 47.45\\
     $0.10$ & 64.08 & 64.67 & 61.33 & 52.01 & 51.25\\
     $0.50$ & \textit{Diverge} & 64.80 & 62.14 & 55.90 & 51.37\\
     $1.00$  & \textit{Diverge} & \textit{Diverge} & \textit{Diverge} & 56.97 & 52.06\\
    \bottomrule
  \end{tabular}
  }
\end{table}

\subsection{Ablation Study}
\label{sec: Ablation Study}

To evaluate the impact of the hyperparameter of the forgetting factor $\beta$ on the performance of the model, we performed a sensitivity analysis in multiple datasets. In this analysis, we measured the accuracy of the test set of the model for various $\beta$ values (see Table \ref{Table: beta ablation}). The results indicate that excessively high values of $\beta$ close to 1, as well as lower values (e.g., 0.85 or below), often lead to suboptimal performance or even divergence. In contrast, values of $\beta$ in the range of 0.95 to 0.99 generally yield better results.

\begin{table}
  \caption{Effect of forgetting factor ($\beta$) on model performance. Bolded values indicate the highest performance in each column.}
  \label{Table: beta ablation}
  \centering
  \scriptsize
  \setlength{\tabcolsep}{2pt} 
  \resizebox{\linewidth}{!}
  {
  \begin{tabular}{|c|c|c|c|c|c|c|}
    \toprule 
     \textbf{Forgetting} & \multirow{2}{*}{\textbf{ImageNet}} & \multirow{2}{*}{\textbf{OxfordPets}} & \multirow{2}{*}{\textbf{Food101}} & \multirow{2}{*}{\textbf{SUN397}} & \multirow{2}{*}{\textbf{DTD}} & \multirow{2}{*}{\textbf{EuroSAT}} \\ 
     \textbf{Factor ($\beta$)} &   &   &   &    &   &  \\ 
     \toprule 
     $0.99$ & \textbf{65.47} & 85.24 & \textbf{81.24} & 70.30 & \textbf{67.12} & \textit{Diverge} \\
     $0.98$ & 65.30 & \textbf{89.23} & 80.06 & 72.95 & 66.56 & 83.89\\
     $0.95$ & 63.42 & 85.35 & 79.48 & \textbf{74.08} & 64.18 & 84.60\\
     $0.90$ & 63.25 & 88.57 & 79.00 & 73.90 & 64.23 & 85.30\\
     $0.85$ & 63.02 & 86.53 & 79.02 & 72.80 & 64.56 & \textbf{87.00}\\
     $0.80$  & 62.89 & 84.16 & 76.38 & 71.45 & 66.10 & 86.68\\
    \bottomrule
  \end{tabular}
  }
\end{table}

\section{Conclusion}
\label{sec: Conclusion}
In this study, we present a Kalman-based algorithm for fine-tuning vision-language pre-trained models like CLIP. 
Our Kalman-based optimization algorithm closely approximates the natural gradient direction within a Bayesian framework. While natural gradient facilitates improved ID performance, Bayesian formulation inherently enables uncertainty quantification, which leads to improvement in OOD generalization.
Extensive experiments are done on various image classification datasets in different ID and OOD scenarios for various few-shot fine-tuning setups. The empirical evidence shows that our algorithm consistently achieves superiority over--or comparability with--baseline methods.
Although this paper focuses on the vision-language classification task, extending the proposed method to other vision-language tasks, such as image captioning, visual question answering, and text-to-image generation, is left for future work.

\bibliography{bibliography.bib}
\bibliographystyle{plain}

\end{document}